\documentclass{article}

\usepackage{arxiv}

\usepackage[utf8]{inputenc} 
\usepackage[T1]{fontenc}    
\usepackage{hyperref}       
\usepackage{url}            
\usepackage{booktabs}       
\usepackage{nicefrac}       
\usepackage{lipsum}

\usepackage{amsmath}
\usepackage{amssymb}

\usepackage{amsthm}
\newtheorem{theorem}{Theorem}
\newtheorem{lemma}{Lemma}
\newtheorem{remark}{Remark}

\usepackage{chngcntr}
\counterwithin*{lemma}{section}
\counterwithin*{theorem}{section}
\counterwithin*{remark}{section}

\usepackage{biblatex}
\title{A Special Case of Quadratic Extrapolation Under the Neural Tangent Kernel}

\author{
 Abiel J. Kim \\
}

\begin{document}
\maketitle
\begin{abstract}
It has been demonstrated both theoretically and empirically that the ReLU MLP tends to extrapolate linearly for an out-of-distribution evaluation point. The machine learning literature provides ample analysis with respect to the mechanisms to which linearity is induced. However, the analysis of extrapolation at the origin under the NTK regime remains a more unexplored special case. In particular, the infinite-dimensional feature map induced by the neural tangent kernel is not translationally invariant. This means that the study of an out-of-distribution evaluation point very far from the origin is not equivalent to the evaluation of a point very near the origin. And since the feature map is rotation invariant, these two special cases may represent the most canonically extreme bounds of ReLU NTK extrapolation. Ultimately, it is this loose recognition of the two special cases of extrapolation that motivate the discovery of quadratic extrapolation for an evaluation close to the origin.
\end{abstract}


\section{Introduction}

The work of \textcite{DBLP:journals/corr/abs-2009-11848} proves that an over-parameterized ReLU-activated multilayer perceptron (MLP) will extrapolate linearly when evaluated along any direction very distant from the origin. They formally prove extrapolative linearity by analysis of the learned regressor's functional form in the \textit{neural tangent kernel} (NTK) \textit{reproducing kernel hilbert space} (RKHS) \parencite{DBLP:journals/corr/abs-1806-07572}. And, since the infinite dimensional feature map induced by the neural tangent kernel is rotation invariant, the analysis covers the generalizable case of an evaluation point very distant from the origin. However, it is not difficult to recognize that the same feature map is not translation invariant. It is by a geometric reasoning that the origin of the RKHS must be a distinct special case whose analysis departs from Theorem 1 of \textcite{DBLP:journals/corr/abs-2009-11848}. That is, in the limit of a large relative distance between the training point set and the evaluation point, one observes that there must be two special locations of the evaluation point with respect to the NTK induced feature map: A location casted along a singular feature direction, and a location which intersects all feature directions.

It is this recognition of the distinguishable cases that motivates the extrapolative analysis at the origin location. The non translation invariance of the feature map implies that the extrapolative analysis at the origin and far from origin are not equivalent problems. It can be reasoned that they are two canonical cases of a more complete analysis of extrapolation. However, inducing extrapolation at the origin must be done carefully to ensure that the evaluation data is pushed out of the support of the training distribution space. This is achieved by this paper's definition of a labeled training set, which is formally presented in the problem setup of section 2. The desired effect of said definition is to induce a problem setup where all members of the training set are sent infinitely far away from the origin whilst fixing the evaluation data at the origin. Under this variant setting, we state Theorem 1, which discovers that an over-parameterized neural network extrapolates quadratically when evaluated near the origin. This finding contrasts, but does not conflict with, \textcite{DBLP:journals/corr/abs-2009-11848}, which contrastingly concerns itself an evaluation point far from the origin.

The paper is organized as follows. The proof of Theorem 1 is presented in §A.4 and will depend on the results of Lemmas 1 and 2, which are proven with continuity in §A.2 and §A.3 respectively. Our problem setup induces a special case of the NTK gram matrix which must be studied in §A.1 to set the stage for the remainder of the mathematics.


\section{Preliminaries}

\textbf{Background on NTK:} Suppose that a neural network performs nonlinear regression $f(\boldsymbol{\theta}, \boldsymbol{x}) : \mathcal{X} \to \mathbb{R}$ where $\boldsymbol{\theta}$ is a vectorization of the network parameters and $\boldsymbol{x} \in \mathcal{X}$. Let there be $n$ training points which form a labeled set $\{ (\boldsymbol{x}_i, y_i) \}_{i=1}^n$. If we train the network on the labeled set to minimize the squared loss function $\frac{1}{2} \sum_{i=1}^n \left( f(\boldsymbol{\theta}, \boldsymbol{x}_i) - y_i \right)^2$ via gradient descent, then we can derive a kernel method from the network by first considering an affine approximation of the network output in parameter space. If we denote the time-dependent parameter vector induced by gradient descent as $\boldsymbol{\theta}^{(t)}$ for some iteration $t$, then we define the feature map $\phi(\boldsymbol{x})$ as the gradient of the network output with respect to $\boldsymbol{\theta}$ evaluated at $\boldsymbol{\theta}^{(0)}$ denoted as $\nabla_{\boldsymbol{\theta}} f(\boldsymbol{\theta}^{(0)}, \boldsymbol{x})$. The corresponding kernel, called the \textit{neural tangent kernel} (NTK), is then an affine model that is linear in the network parameters. Under particular constraints such as the infinite width and infinitesimal learning rate, the NTK becomes an expectation:
\begin{align*}
    NTK(\boldsymbol{x}_i, \boldsymbol{x}_j) = 
    \mathbb{E}_{\boldsymbol{\theta} \sim \mathcal{N}}
    \left\langle 
        \nabla_{\boldsymbol{\theta}} f(\boldsymbol{\theta}^{(0)}, \boldsymbol{x}_i),
        \nabla_{\boldsymbol{\theta}} f(\boldsymbol{\theta}^{(0)}, \boldsymbol{x}_j)
    \right\rangle,
\end{align*}
where the expectation emerges by the law of large numbers induced by the network's infinite width. Interestingly, the affine approximation is correct under the NTK constraints in parameter space, and is closely tied to the network's notion of \textit{lazy training}. Ultimately, since training is linear in the often high-dimensional, possibly infinite, feature space, the neural network behaves as an affine kernel regression. We take all such pairwise NTK evaluations from the labeled training set to produce the positive semi-definite NTK gram matrix denoted as $NTK_{train}$.

\textbf{Background on Neural Network Extrapolation:} \textcite{DBLP:journals/corr/abs-2009-11848} builds on the established results of the NTK equivalence between neural network training and kernel regression to more precisely analyze extrapolation. However, using the NTK directly requires analysis of the point-wise form as a kernel regression fit over the labeled training set. It can be more advantageous to work in the NTK induced feature space instead to derive a functional representation of the learned network, which may be more analytically manageable. This is precisely the route they take and formalize this equivalence between point-wise NTK regression and the learned function in the NTK induced feature space in their Lemma 2:
\begin{align*}
    f_{NTK}(\boldsymbol{x}) 
    = \phi(\boldsymbol{x})^\top \boldsymbol{\beta}_{\text{NTK}} \\
    \text{where} \quad \boldsymbol{\beta}_{\text{NTK}} &= \min_{\boldsymbol{\beta}} \|\boldsymbol{\beta}\|_2 \\
    \text{s.t.} \quad & \phi(\boldsymbol{x})^\top \boldsymbol{\beta} = y_i \quad \text{for } i=1,\dots,n,
\end{align*}
where $f_{NTK}(\boldsymbol{x}) = \phi(\boldsymbol{x})^\top \boldsymbol{\beta}_{\text{NTK}}$ is the min-norm functional form equivalent to NTK kernel regression fitted over the training data for any $\boldsymbol{x} \in \mathcal{X}$. Further, they derive the precise closed-form of the NTK induced feature map for a ReLU two-layer MLP in their Lemma 3:
\begin{align*}
    \phi(\boldsymbol{x}) 
    = c' 
    \Big( 
        \boldsymbol{x} 
        \cdot \mathbb{I}\! \left( {\boldsymbol{w}^{(k)}}^\top \boldsymbol{x} \geq 0 \right), \,
        {\boldsymbol{w}^{(k)}}^\top \boldsymbol{x} \cdot \mathbb{I}\!\left( {\boldsymbol{w}^{(k)}}^\top \boldsymbol{x} \geq 0 \right), \,
        \ldots
    \Big),
\end{align*}
where $c'$ is a constant, $\mathbb{I}$ is the Heaviside indicator function, and $\boldsymbol{w}^{(k)} \sim \mathcal{N}(\boldsymbol{0}, \boldsymbol{I})$ for $k \to \infty$. By analyzing the functional representation in the NTK RKHS, they discovered that for a labeled training set $\{ ({\boldsymbol{x}}_i, y_i)\}_{i=1}^n$ and evaluation point $\boldsymbol{x}_0 = t \boldsymbol{v}$ for any direction $\boldsymbol{v} \in \mathbb{R}^d$, the network converges to a linear function.

\textbf{Problem Setup:} Our problem setup inherits from both \textcite{DBLP:journals/corr/abs-1806-07572} and \textcite{DBLP:journals/corr/abs-2009-11848} primarily through the notation of the latter for compatibility. Let $\mathcal{X}$ be a $d$-dimensional Euclidean input space and $\varphi$ be a set of $n$ training inputs such that $\varphi = \{ \boldsymbol{x}_i \}_{i=1}^n$ with $\boldsymbol{x}_i \in \mathcal{X}$ for $i \in [n]$. If we translate $\varphi$  by the vector $- t \boldsymbol{v}_\varphi$ for any direction $\boldsymbol{v}_\varphi$, then we will have formed a new set $\varphi^\infty = \{ \boldsymbol{x}_i - t \boldsymbol{v}_\varphi : \boldsymbol{x}_i  \in \varphi \}$ where a member is denoted $\boldsymbol{x}_i^\infty = \boldsymbol{x}_i - t \boldsymbol{v}_\varphi$ for any $\boldsymbol{x}_i^\infty \in \varphi^\infty$. The labeled training set can then be constructed as $\{ (\boldsymbol{x}_i^\infty, y_i^\infty) \}_{i=1}^n$ where $y_i^\infty = g(\boldsymbol{x}_i^\infty)$ for target function $g : \mathcal{X} \to \mathbb{R}$. We train a single-output two-layer ReLU MLP $f_{NTK} : \mathcal{X} \to \mathbb{R}$ in the NTK regime using gradient descent to minimize the squared loss function over the labeled training set. We reintroduce the hat notation which denotes a data vector augmented with bias term: $\hat{\boldsymbol{x}} = \left[ \boldsymbol{x} | 1 \right]$. We also introduce the check notation which denotes the explicit exclusion of the bias weight with respect to the $k$-th hidden neuron, $\check{\boldsymbol{w}}^{(k)} \in \mathbb{R}^d$.

\textit{Clarifying the Training Data:} Please agree that the definition of the labeled training set, which is constructed from $\varphi^\infty$, facilitates an analysis of extrapolation at the origin location. If extrapolation can be configured by defining the labeled training set $\{ ({\boldsymbol{x}}_i, y_i)\}_{i=1}^n$ and the evaluation point $\boldsymbol{x}_0 = t \boldsymbol{v}$ for any direction $\boldsymbol{v}$, then it is not difficult to attain the ``inverted'' setup of the new training set $\{ (\boldsymbol{x}_i-t\boldsymbol{v}, \, g(\boldsymbol{x}_i-t\boldsymbol{v}))\}_{i=1}^n$ and evaluation point $\boldsymbol{x}_0 = \boldsymbol{0}$. Critically, the coordinate shift preserves the sufficient condition that induces extrapolation as $t \to \infty$.

\textit{Clarifying the Notation:} This paper refers to the set $\varphi$ as a (point) realization, which is a convention related to point processes and stochastic geometry \parencite{chiu2013stochastic}. Since the NTK regime deals with finite datasets, it may be useful to explicitly describe or analyze $\varphi$ as a point process in a later related work. The set $\varphi$ may be ascribed an underlying mathematical data generator for the purposes of a neural scaling law analysis, for instance.

\subsection{Related Work}

This paper identifies and deeply explores a special case of nonlinear NTK extrapolation. To the best of our knowledge, since our work is based on \textcite{DBLP:journals/corr/abs-2009-11848} as a special coordinate-shifted case of their problem setup, discovering additional relevant literature in the space of NTK extrapolation is challenging. However, there are previous works that strongly align with the themes of this paper insofar as the exploration of special nonlinear regimes or NTK configurations. One notable work is \textcite{bai2020linearizationquadratichigherorderapproximation} where they discover a special learning process using randomization that results in a dominant quadratic Taylor term as opposed to the standard linear dominance in a Taylor expansion. But, it must be made clear that the results of \textcite{bai2020linearizationquadratichigherorderapproximation} do not specifically address extrapolation. 

Furthermore, various elements of this manuscript align with existing work insofar as the application of mathematical techniques for machine learning. We consider, for instance, how \textcite{rangamani2020interpolatingkernelmachinesminimizing}'s Remark 3 justifies our usage of Tikhononv regularization to pseudo-invert a special geometrically constrained NTK gram matrix in §A.1. Ultimately, this paper is an analysis of asymptotic quadratic extrapolation for over-parameterized neural networks and serves a complementary work to \textcite{DBLP:journals/corr/abs-2009-11848}.

\section{Theoretical Contributions}

\begin{remark}
    If all $n$ inputs of the training set $\varphi^\infty$ are located infinitely far from the origin along the same direction, then the asymptotic pseudo-inverse of the NTK gram matrix is a difference between the identity and all-ones matrix: $\frac{1}{\delta} \boldsymbol{I} - \frac{t^2 \kappa}{\delta(n \kappa t^2 + \delta)} \boldsymbol{J}$, where $\delta \to 0$, $t \to \infty$, and $\kappa$ is a constant.
\end{remark}
\textit{Special Case of the NTK Gram:} We discover this closed-form in §A.1 by first recognizing that under the definition of $\varphi^\infty$, the indicators for any training input become \textit{input agnostic} insofar that the indicating logic strictly depends on a feature direction $\boldsymbol{w}$ and training direction $\boldsymbol{v}_{\varphi}$. The definition of $\varphi^\infty$ induces the otherwise singular asymptotic NTK gram matrix $\kappa t^2 \boldsymbol{J}$. We then use Tikhonov regularization to pseudo-invert this asymptotic NTK gram expressed as $\left( \kappa t^2 \boldsymbol{J} + \boldsymbol{\Gamma} \right)^{-1}$. We leverage this special case of the asymptotic NTK gram matrix induced by $\varphi^\infty$ and its pseudo-inverse to express the components of $\boldsymbol{\beta}_{NTK}$ induced by training inputs $\varphi^\infty$ that are distant from the origin. These results are then used to prove Lemma 2 and ultimately Theorem 1.

\begin{theorem}
An over-parameterized two-layer ReLU MLP $f_{NTK} : \mathbb{R}^{d} \to \mathbb{R}$ that is trained on a labeled set $\{ (\boldsymbol{x}_{i}^\infty, {y}_{i}^\infty) \}_{i=1}^{n}$ with $\boldsymbol{x}_{i}^\infty = \boldsymbol{x}_i - t \boldsymbol{v}_{\varphi}$ for $\boldsymbol{x}_i \in \mathcal{X}$ and any direction $\boldsymbol{v}_{\varphi}$ in the NTK regime minimizing squared loss will converge to a quadratic extrapolator when evaluated at a point near the origin $\boldsymbol{0}$ as $t \to \infty$.
\end{theorem}
\textit{Theorem 1 Proof Sketch:} Theorem 1 is the main contribution of this paper and states that an extremely wide NTK predictor with ReLU activations that is trained on a dataset which is extremely distant from the origin will converge to a quadratic extrapolator when evaluated near the origin. That is, the Theorem 1 states that the predictor's first and second directional derivatives exist and all higher-order derivatives are $0$. And the proof of Theorem 1, which is presented in §A.4, depends on the results of Lemmas 1 and 2. Lemma 1 is a generalized algebraic manipulation and states that the directional derivative of the NTK predictor can be expressed in terms of the derivatives of the indicator. The significance of Lemma 1 is most clear when we leverage the Dirac-delta's so called \textit{sifting property}, also known as the \textit{sampling property}. We note that the derivative of the Heaviside indicator is the Dirac-delta, and applies itself nicely when viewing the predictor's derivative as an integral. Lemma 2 completes the second half of the Theorem 1 proof by stating that the partial derivatives of the beta components with respect to the bias component of a feature direction $\boldsymbol{w}_{d+1}$ are always $0$ for any order derivative past the second. The significance of Lemma 2 is clear when we see in §A.2 that the $z$-th derivative of the predictor depends on the $(z-1)$-th and $(z-2)$-th partial derivatives of the beta components. It is not difficult to see that the quadratic-order persists when taking the $z$-th derivative of $f_{NTK}$.
\begin{lemma}
The feature map of the $z$-th directional derivative of $f_{NTK}$ for any direction $\boldsymbol{v}_{0}$ can be expressed in terms of the $z$-th and $(z-1)$-th directional derivatives of the indicator for $\boldsymbol{v}_{0}$ such that:
\begin{align*}
    &D^{z}_{\boldsymbol{v}} f_{NTK}(\boldsymbol{x}_0)
    = \boldsymbol{\beta}_{\text{NTK}}^\top
    \left(
        \boldsymbol{\hat{x}}_0 \cdot D_{\boldsymbol{v}}^{z} \, \mathbb{I}^{(k)}
        - z \hat{\boldsymbol{v}} \cdot D_{\boldsymbol{v}}^{z-1} \, \mathbb{I}^{(k)}, \,
        {\boldsymbol{w}^{(k)}}^\top \boldsymbol{\hat{x}}_0 \cdot D_{\boldsymbol{v}}^{z} \, \mathbb{I}^{(k)}
        - z {\boldsymbol{w}^{(k)}}^\top \hat{\boldsymbol{v}} \cdot D_{\boldsymbol{v}}^{z-1} \, \mathbb{I}^{(k)}, 
        \, ...
    \right)
\end{align*}
\end{lemma}

\begin{lemma}
The components of the NTK representation coefficient $\boldsymbol{\beta}_{NTK}$ induced by a training input set $\varphi^\infty = \{ \boldsymbol{x}_{i}^\infty \}_{i=1}^{n}$ where $\boldsymbol{x}_{i}^\infty = \boldsymbol{x}_i - t \boldsymbol{v}_{\varphi}$ for some $\boldsymbol{x}_i \in \mathcal{X}$ and any direction $\boldsymbol{v}_{\varphi}$ are constant with respect to the bias component of any given feature direction $\boldsymbol{w}_{d+1}$ such that:
\begin{align*}
    \frac{\partial^z \boldsymbol{\beta}^1_{\boldsymbol{w}}}{\partial \boldsymbol{w}^z_{d+1}},
    \frac{\partial^z \boldsymbol{\beta}^2_{\boldsymbol{w}}}{\partial \boldsymbol{w}^z_{d+1}}
    \to 0 
    \; \text{for all} \; z \ge 1.
\end{align*}
\end{lemma}

\section{Conclusion}

This paper identifies a special case of nonlinearity for NTK extrapolation at the origin of the RKHS. More specifically, this paper finds that at the origin, the infinitely-wide two-layer MLP retains a non-zero second-order Taylor term in the limit. The quadratic behavior is highly dependent on the degree of similarity between vector orientations $\boldsymbol{w}$ - which represents the direction of a feature - and $\boldsymbol{v}$ - the vector which defines the evaluation point. If, for instance, the two orientations are orthogonal, then the second derivative is unconditionally zero. The second derivative may also be zero dependent on the beta components, i.e., if the beta 1 component is orthogonal to $\boldsymbol{v}$ and the beta 2 component is zero; However, this condition is less strict. The results are distinct from but complementary to the existing ML literature which primarily concern the linearity of neural network extrapolation. That is, since the feature map induced by the neural tangent kernel is not translation invariant, extrapolation at a point far from the origin is not equivalent to extrapolation a point close to the origin. We prove our results by determining a closed-form of the asymptotic pseudo-inverse NTK gram matrix to determine the components of $\boldsymbol{\beta}_{NTK}$ induced by the definition of $\varphi^\infty$. Then we discover a neat algebraic trick in Lemma 1 to rewrite the directional derivative of the predictor as partial derivatives of the beta components using the distributional derivative equivalence to the directional derivative of the indicator.

\section*{Acknowledgments}
In accordance with the arXiv policy for author usage of generative AI language tools, we report the use of Gemini 2.5 as a first line defense tool to primarily check for potential mistakes in mathematical derivations. The language and format of this manuscript was generally written and prepared without the assistance of a generative AI. The author asserts complete leadership over the formal proofs of this manuscript with all usage of generative AI limited to serving at an assistive capacity. The author hereby assumes full responsibility for the content of this manuscript.

\nocite{*}
\printbibliography

\clearpage

\appendix
\section{Proofs}

\subsection{Special Case of the NTK Gram Matrix}

We begin our analysis by making clear the form of $\boldsymbol{\beta}$, which is the coefficient vector in the NTK RKHS that is fit over the labeled training data. We begin with the point-wise form of NTK regression to write $\boldsymbol{\beta}$ in terms of the NTK gram:
\begin{align}
    & f_{NTK}(\boldsymbol{\hat{x}}) 
    = \big(
        \langle \phi (\boldsymbol{\hat{x}}), \phi (\boldsymbol{\hat{x}}^\infty_1) \rangle, \ldots,
        \langle \phi (\boldsymbol{\hat{x}}), \phi (\boldsymbol{\hat{x}}^\infty_n) \rangle
    \big)^\top
    \cdot \boldsymbol{NTK}_{train}^{-1} \boldsymbol{Y} \\
    &= \phi(\boldsymbol{\hat{x}})^\top \boldsymbol{\Phi}_{train}^\top \boldsymbol{NTK}^{-1}_{train} \boldsymbol{Y} \\
    &= \phi(\boldsymbol{\hat{x}})^\top \boldsymbol{\beta}.
\end{align}
Attaining a closed form expression of $\boldsymbol{NTK}^{-1}_{train}$ is a desirable but non-trivial analysis. Fortunately, later in this section, we will see how the definition of $\varphi^\infty$ induces the NTK gram to a closed-form asymptotic pseudo-inverse. But first, we recognize the application of Tikhonov regularization, which ensures the invertibility of the NTK gram matrix and induces a choice of $\boldsymbol{\beta}$ equivalent to the min-norm definition of the unique $\boldsymbol{\beta}_{NTK}$. Tikhonov regularization was chosen for its simple usage but is also an approach supported by \textcite{rangamani2020interpolatingkernelmachinesminimizing}. We express $\boldsymbol{\beta}_{NTK}$ in terms of the Tikhonov regularized NTK gram matrix: 
\begin{align}
    \boldsymbol{\beta}_{NTK} = 
    \boldsymbol{\Phi}_{train}^\top
    \big(
        \boldsymbol{NTK}_{train} + \boldsymbol{\Gamma}
    \big)^{-1} 
    \boldsymbol{Y}
\end{align}
for Tikhonov matrix $\boldsymbol{\Gamma} = \delta \boldsymbol{I}, \delta \to 0^+$. Before we solve for the pseudo-inverse of $\boldsymbol{NTK}_{train}$, we take note of the induced behavior of the indication function for a training data point under the definition of $\varphi^\infty$, where we find that the indication depends solely on the dot product between any particular feature direction $\boldsymbol{w}$ and the special $\boldsymbol{v}_\varphi$ that translates $\varphi$. In other words, under the definition of $\varphi^\infty$, ReLU indicators for any training data point $\boldsymbol{x}_i^\infty \in \varphi^\infty$ become input agnostic insofar that they become independent from $\boldsymbol{x}_i \in \varphi$:
\begin{align}
    & \mathbb{I} \left( {\boldsymbol{w}}^\top \boldsymbol{\hat{x}}^\infty_i \geq 0 \right) \\
    &= \mathbb{I} \left( {\boldsymbol{w}}^\top (\boldsymbol{\hat{x}}_i - t\boldsymbol{\hat{v}}) \geq 0 \right) \\
    &= \mathbb{I} \left( {\boldsymbol{w}}^\top (-\boldsymbol{\hat{v}}) \geq 0 \right).
\end{align}
The independence that arises between the indicators and the training inputs is a crucial insight and will be a recurring assistance that enables the pseudo-inversion of the NTK gram. Speaking of which, by definition of the neural tangent kernel, the $(i,j)$-th entry of $\boldsymbol{NTK}_{train}$ can be expressed as:
\begin{align}
    & \boldsymbol{NTK}_{train}[i, j] = \left< \phi(\boldsymbol{\hat{x}}^\infty_i), \phi(\boldsymbol{\hat{x}}^\infty_j) \right> \\
    &= \int \boldsymbol{\hat{x}}^\infty_i \cdot \boldsymbol{\hat{x}}^\infty_j \cdot
    \mathbb{I} \left( {\boldsymbol{w}}^\top \boldsymbol{\hat{x}}^\infty_i \geq 0 \right) \cdot
    \mathbb{I} \left( {\boldsymbol{w}}^\top \boldsymbol{\hat{x}}^\infty_j \geq 0 \right) \\
    &+ ({\boldsymbol{w}^{(k)}}^\top \boldsymbol{\hat{x}}^\infty_i) \cdot 
    ({\boldsymbol{w}^{(k)}}^\top \boldsymbol{\hat{x}}^\infty_j) \cdot
    \mathbb{I} \left( {\boldsymbol{w}}^\top \boldsymbol{\hat{x}}^\infty_i \geq 0 \right) \cdot
    \mathbb{I} \left( {\boldsymbol{w}}^\top \boldsymbol{\hat{x}}^\infty_j \geq 0 \right) d \mathbb{P}(\boldsymbol{w})
\end{align}
for any pair $( \boldsymbol{x_i^\infty}, \boldsymbol{x_j^\infty} )$ taken from the labeled training set. We observe the emergence of an indication pair in lines (9)-(10). But since indicators become input agnostic, we greatly simplify their indicating logic using equation (7):
\begin{align}
    &   \mathbb{I} \left( {\boldsymbol{w}}^\top \boldsymbol{\hat{x}}^\infty_i \geq 0 \right) \cdot
        \mathbb{I} \left( {\boldsymbol{w}}^\top \boldsymbol{\hat{x}}^\infty_j \geq 0 \right) \\
    &=  \mathbb{I} \left( {\boldsymbol{w}}^\top \boldsymbol{\hat{x}}_i - t ({\boldsymbol{w}}^\top \boldsymbol{\hat{v}}) \geq 0 \right) \cdot
        \mathbb{I} \left( {\boldsymbol{w}}^\top \boldsymbol{\hat{x}}_j - t ({\boldsymbol{w}}^\top \boldsymbol{\hat{v}}) \geq 0 \right) \\
    &= \mathbb{I} \left( {\boldsymbol{w}}^\top (-\boldsymbol{\hat{v}}) \geq 0 \right)^2 \\
    &= \mathbb{I} \left( {\boldsymbol{w}}^\top (-\boldsymbol{\hat{v}}) \geq 0 \right).
\end{align}
Then we apply the definition of $\varphi^\infty$ to expand the dot product in equation (9):
\begin{align}
    & \boldsymbol{\hat{x}}^\infty_i \cdot \boldsymbol{\hat{x}}^\infty_j \\
    &= (\boldsymbol{\hat{x}}_i - t \boldsymbol{\hat{v}}) \cdot (\boldsymbol{\hat{x}}_j - t \boldsymbol{\hat{v}}) \\
    &= \boldsymbol{\hat{v}}^2 t^2
    - (\boldsymbol{\hat{x}}_i + \boldsymbol{\hat{x}}_j) \cdot \boldsymbol{\hat{v}} t
    + \boldsymbol{\hat{x}}_i \cdot \boldsymbol{\hat{x}}_j,
\end{align}
as well as the dot product in equation (10):
\begin{align}
    & ({\boldsymbol{w}}^\top \boldsymbol{\hat{x}}^\infty_i) \cdot 
    ({\boldsymbol{w}}^\top \boldsymbol{\hat{x}}^\infty_j) \\
    &= ({\boldsymbol{w}}^\top (\boldsymbol{\hat{x}}_i - t \boldsymbol{\hat{v}})) \cdot 
    ({\boldsymbol{w}}^\top (\boldsymbol{\hat{x}}_j - t \boldsymbol{\hat{v}}) \\
    &= ({\boldsymbol{w}}^\top \boldsymbol{\hat{v}})^2 t^2 
    - {\boldsymbol{w}}^\top \boldsymbol{\hat{v}} ({\boldsymbol{w}}^\top \boldsymbol{\hat{x}}_i + {\boldsymbol{w}}^\top \boldsymbol{\hat{x}}_j) t
    + ({\boldsymbol{w}}^\top \boldsymbol{\hat{x}}_i) \cdot ({\boldsymbol{w}}^\top \boldsymbol{\hat{x}}_j),
\end{align}
to rewrite the $(i, j)$-th entry of the NTK gram matrix using lines (14), (17), and (20) as:
\begin{align}
    & \boldsymbol{NTK}_{train}[i, j] \\
    &= t^2 \int 
        (\boldsymbol{\hat{v}}^2 + ({\boldsymbol{w}}^\top \boldsymbol{\hat{v}})^2) 
        \cdot \mathbb{I} \left( {\boldsymbol{w}}^\top (-\boldsymbol{\hat{v}}) \geq 0 \right)
    d \mathbb{P}(\boldsymbol{w}) \\
    &- t \int 
        ((\boldsymbol{\hat{x}}_i + \boldsymbol{\hat{x}}_j) \cdot \boldsymbol{\hat{v}} + {\boldsymbol{w}}^\top \boldsymbol{\hat{v}} ({\boldsymbol{w}}^\top \boldsymbol{\hat{x}}_i + {\boldsymbol{w}}^\top \boldsymbol{\hat{x}}_j))
        \cdot \mathbb{I} \left( {\boldsymbol{w}}^\top (-\boldsymbol{\hat{v}}) \geq 0 \right)
    d \mathbb{P}(\boldsymbol{w}) \\
    &+ \int 
        (\boldsymbol{\hat{x}}_i \cdot \boldsymbol{\hat{x}}_j + ({\boldsymbol{w}}^\top \boldsymbol{\hat{x}}_i) \cdot ({\boldsymbol{w}}^\top \boldsymbol{\hat{x}}_j))
        \cdot \mathbb{I} \left( {\boldsymbol{w}}^\top (-\boldsymbol{\hat{v}}) \geq 0 \right)
    d \mathbb{P}(\boldsymbol{w}).
\end{align}
The quadratic form that emerges in lines (22)-(24) is a direct consequence of applying the definition of $\varphi^\infty$; It defines the NTK gram matrix induced by the limiting training set. It is a beautiful structure because the leading-order term is the only term in the quadratic form that does not depend on indices $i$ and $j$. Without this particular structure, pseudo-inverting the matrix $(\boldsymbol{NTK}_{train} + \boldsymbol{\Gamma})$ for a closed-form would be a more difficult analysis. Since line (22) is the leading order term, the resulting form intuitively suggests that as $\varphi$ is shifted further from the origin along some direction $\boldsymbol{v}_\varphi$, the kernel regression solution depends less on the inputs of $\varphi$ and more on the direction $\boldsymbol{v}_\varphi$:
\begin{align}
    &\boldsymbol{NTK}_{train}[i, j] \asymp t^2 \kappa,
\end{align}
where $\kappa$ is a constant equal to the integral of line (22). Therefore, in the limit as $t \to \infty$, we find that the $(i,j)$-th entry of the NTK gram does not depend on $\varphi$. The asymptotic form is then a constant matrix, meaning that $\boldsymbol{NTK}_{train}[i,j]$ is constant for any pair $(i,j)$. We can finally invert the regularized NTK gram from line (4) as:
\begin{align}
    & \left( \boldsymbol{NTK}_{train} + \boldsymbol{\Gamma} \right)^{-1} \\
    &\asymp \left( t^2 \kappa \boldsymbol{J} + \boldsymbol{\Gamma} \right)^{-1} \\
    &= \left( \delta \boldsymbol{I} + (t^2  \boldsymbol{1}) (\kappa \boldsymbol{1})^\top \right)^{-1} \\
    &= \frac{1}{\delta} \boldsymbol{I} - \frac{t^2 \kappa}{\delta(n \kappa t^2 + \delta)} \boldsymbol{J},
\end{align}
where $\boldsymbol{J}[i,j]=1$ for any pair of indices $(i,j)$. The penultimate equality has $\delta \boldsymbol{I}$ from our definition of $\boldsymbol{\Gamma}$ with the outer product between $t^2 \boldsymbol{1}$ and $\kappa \boldsymbol{1}$. In the final equality, one inverts the matrix using the Sherman-Morrison formula \parencite{maponi2007solution}. It follows from line (29) that the $(i,j)$-th entry of the NTK gram asymptotic pseudo-inverse is:
\begin{align}
    \left( \boldsymbol{NTK}_{train} + \boldsymbol{\Gamma} \right)^{-1}[i,j]
    \asymp \begin{cases}
    - \frac{\kappa t^2 }{\delta(n \kappa t^2 + \delta)}, & \text{if } i \ne j \\
    \frac{1}{\delta} - \frac{\kappa t^2 }{\delta(n \kappa t^2 + \delta)}, & \text{if } i = j
    \end{cases}
\end{align}
Using the piecewise definition of equation (30), let $\boldsymbol{\alpha}_{NTK} \asymp \left( \frac{1}{\delta} \boldsymbol{I} - \frac{t^2 \kappa}{\delta(n \kappa t^2 + \delta)} \boldsymbol{J} \right) \boldsymbol{Y}$ denote the matrix-vector product between the label vector $\boldsymbol{Y}$ and the asymptotic pseudo-inverse. Note that $\boldsymbol{\alpha}_{NTK}$ is sub-scripted as such so that the applied regularization is explicit. It is not difficult to calculate the closed-form of the $i$-th entry of $\boldsymbol{\alpha}_{NTK}$:
\begin{align}
    & \boldsymbol{\alpha}_{NTK}[i] = \left( \frac{1}{\delta} \boldsymbol{I} - \frac{t^2 \kappa}{\delta(n \kappa t^2 + \delta)} \boldsymbol{J} \right)[i] \cdot \boldsymbol{Y} \\
    &= \sum^{n}_{j=1} \left( \frac{1}{\delta} \boldsymbol{I} - \frac{t^2 \kappa}{\delta(n \kappa t^2 + \delta)} \boldsymbol{J} \right)[i,j] \cdot \boldsymbol{Y}[j] \\
    &= \sum_{j=1}^n \left( - \frac{t^2 \kappa}{\delta(n \kappa t^2 + \delta)} \right) g(\boldsymbol{\hat{x}}^\infty_j) 
    + \frac{1}{\delta} g(\boldsymbol{\hat{x}}^\infty_i) \\
    &= - \frac{t^2 \kappa}{\delta(n \kappa t^2 + \delta)} \sum_{j=1}^n g(\boldsymbol{\hat{x}}^\infty_j) 
    + \frac{1}{\delta} g(\boldsymbol{\hat{x}}^\infty_i).
\end{align}
And, it should be made clear the values of the $\boldsymbol{\beta}$ components. There are two components associated with a feature direction $\boldsymbol{w}^{(k)}$ for any $k$. We follow the notation of \textcite{DBLP:journals/corr/abs-2009-11848} and denote the first (vector) beta component as $\boldsymbol{\beta}_{\boldsymbol{w}}^{1}$ and the second (scalar) beta component as $\boldsymbol{\beta}_{\boldsymbol{w}}^{2}$ denoting $\boldsymbol{w}$ as a shorthand for any particular $\boldsymbol{w}^{(k)}$. See line (38) below:
\begin{align}
    &= \boldsymbol{\Phi}_{train}^\top \boldsymbol{\alpha} \\
    &= \boldsymbol{\alpha}_1 \phi (\boldsymbol{\hat{x}}^\infty_1)
    + \boldsymbol{\alpha}_2 \phi (\boldsymbol{\hat{x}}^\infty_2)
    + ...
    + \boldsymbol{\alpha}_n \phi (\boldsymbol{\hat{x}}^\infty_n) \\
    &= \boldsymbol{\alpha}_1 
    \begin{bmatrix}
         \boldsymbol{\hat{x}}^\infty_1 \cdot 
         \mathbb{I} \left( {\boldsymbol{w}^{(k)}}^\top \boldsymbol{\hat{x}}^\infty_1 \geq 0 \right)
         \\
         {\boldsymbol{w}^{(k)}}^\top \boldsymbol{\hat{x}}^\infty_1 \cdot 
         \mathbb{I} \left( {\boldsymbol{w}^{(k)}}^\top \boldsymbol{\hat{x}}^\infty_1 \geq 0 \right)
         \\
         \vdots
    \end{bmatrix}
    + ...
    + \boldsymbol{\alpha}_n
    \begin{bmatrix}
         \boldsymbol{\hat{x}}^\infty_n \cdot 
         \mathbb{I} \left( {\boldsymbol{w}^{(k)}}^\top \boldsymbol{\hat{x}}^\infty_n \geq 0 \right)
         \\
         {\boldsymbol{w}^{(k)}}^\top \boldsymbol{\hat{x}}^\infty_n \cdot 
         \mathbb{I} \left( {\boldsymbol{w}^{(k)}}^\top \boldsymbol{\hat{x}}^\infty_n \geq 0 \right)
         \\
         \vdots
    \end{bmatrix} \\
    &= \begin{bmatrix}
         \boldsymbol{\alpha}_1 \boldsymbol{\hat{x}}^\infty_1 
         \mathbb{I} \left( {\boldsymbol{w}^{(k)}}^\top \boldsymbol{\hat{x}}^\infty_1 \geq 0 \right)
         + ... +
         \boldsymbol{\alpha}_n \boldsymbol{\hat{x}}^\infty_n
         \mathbb{I} \left( {\boldsymbol{w}^{(k)}}^\top \boldsymbol{\hat{x}}^\infty_n \geq 0 \right)
         \\
         \boldsymbol{\alpha}_1 {\boldsymbol{w}^{(k)}}^\top \boldsymbol{\hat{x}}^\infty_1 
         \mathbb{I} \left( {\boldsymbol{w}^{(k)}}^\top \boldsymbol{\hat{x}}^\infty_1 \geq 0 \right)
         + ... +
         \boldsymbol{\alpha}_n {\boldsymbol{w}^{(k)}}^\top \boldsymbol{\hat{x}}^\infty_n
         \mathbb{I} \left( {\boldsymbol{w}^{(k)}}^\top \boldsymbol{\hat{x}}^\infty_n \geq 0 \right)
         \\
         \boldsymbol{\alpha}_1 \boldsymbol{\hat{x}}^\infty_1 
         \mathbb{I} \left( {\boldsymbol{w}^{(k+1)}}^\top \boldsymbol{\hat{x}}^\infty_1 \geq 0 \right)
         + ... +
         \boldsymbol{\alpha}_n \boldsymbol{\hat{x}}^\infty_n
         \mathbb{I} \left( {\boldsymbol{w}^{(k+1)}}^\top \boldsymbol{\hat{x}}^\infty_n \geq 0 \right)
         \\
         \boldsymbol{\alpha}_1 {\boldsymbol{w}^{(k+1)}}^\top \boldsymbol{\hat{x}}^\infty_1 
         \mathbb{I} \left( {\boldsymbol{w}^{(k+1)}}^\top \boldsymbol{\hat{x}}^\infty_1 \geq 0 \right)
         + ... +
         \boldsymbol{\alpha}_n {\boldsymbol{w}^{(k+1)}}^\top \boldsymbol{\hat{x}}^\infty_n
         \mathbb{I} \left( {\boldsymbol{w}^{(k+1)}}^\top \boldsymbol{\hat{x}}^\infty_n \geq 0 \right)
         \\
         \vdots
    \end{bmatrix}.
\end{align}
It follows from line (38) that the components of $\boldsymbol{\beta}_{NTK}$ can be written as:
\begin{align}
    & \boldsymbol{\beta}^1_{\boldsymbol{w}} = \sum_{i=1}^n \boldsymbol{\alpha}_{NTK}[i] \, \boldsymbol{\hat{x}}^\infty_i \,
    \mathbb{I} \left( {\boldsymbol{w}}^\top \boldsymbol{\hat{x}}^\infty_i \geq 0 \right) \\
    & \boldsymbol{\beta}^2_{\boldsymbol{w}} = \sum_{i=1}^n \boldsymbol{\alpha}_{NTK}[i] \, {\boldsymbol{w}}^\top \boldsymbol{\hat{x}}^\infty_i \,
    \mathbb{I} \left( {\boldsymbol{w}}^\top \boldsymbol{\hat{x}}^\infty_i \geq 0 \right).
\end{align}
Lastly, we use lines (7), (34), and the definition of $\varphi^\infty$, to finally rewrite equations (39)-(40) for a closed-form of the first and second beta components that are induced by the definition of $\varphi^\infty$:
\begin{align}
    &\boldsymbol{\beta}^1_{\boldsymbol{w}} =
    \mathbb{I} \left( {\boldsymbol{w}}^\top (-\boldsymbol{\hat{v}}) \geq 0 \right)
    \cdot 
    \left(
        C(t, \delta, \kappa)
        \sum_{j=1}^n g(\boldsymbol{\hat{x}}_i^\infty)
        \sum_{i=1}^n \boldsymbol{\hat{x}}_i^\infty
        +
        \frac{1}{\delta}
        \sum_{i=1}^n \boldsymbol{\hat{x}}_i^\infty g(\boldsymbol{\hat{x}}_i^\infty)
    \right) \\
    &\boldsymbol{\beta}^2_{\boldsymbol{w}} =
    \mathbb{I} \left( {\boldsymbol{w}}^\top (-\boldsymbol{\hat{v}}) \geq 0 \right)
    \cdot 
    \left(
        C(t, \delta, \kappa)
        \sum_{j=1}^n g(\boldsymbol{\hat{x}}_i^\infty)
        \sum_{i=1}^n \boldsymbol{w}^\top \boldsymbol{\hat{x}}_i^\infty
        +
        \frac{1}{\delta}
        \sum_{i=1}^n \boldsymbol{w}^\top \boldsymbol{\hat{x}}_i^\infty g(\boldsymbol{\hat{x}}_i^\infty)
    \right)
\end{align}
where $C(t, \delta, \kappa) = - \frac{t^2 \kappa}{\delta(n \kappa t^2 + \delta)}$ with $t \to \infty$, $\delta \to 0^+$, and $\kappa$ depending on $\boldsymbol{w}$. One final note as an aside is that we can digress into a separate but related analysis if we take the target $g$ to be linear, i.e., we can apply the equivalence $g(\boldsymbol{\hat{x}}_i^\infty) = g(\boldsymbol{\hat{x}}_i) - t g(\boldsymbol{\hat{v}}_\varphi)$ to lines (41)-(42) and analyze unrelated forms. Although this is irrelevant to the paper, it may lead to an alternate proof of somewhat interesting findings.

\subsection{Proof of Lemma 1}

If $\boldsymbol{x}_0 \in \mathcal{X}$ is an evaluation point then let $\boldsymbol{x}_1 = \boldsymbol{x}_0 + h\boldsymbol{v}$ for some direction $\boldsymbol{v}$. We can compute the $z$-th directional derivative of $f_{NTK}$ recursively using the standard limit definition:
\begin{align}
    D^{z}_{\boldsymbol{v}} f_{NTK}(\boldsymbol{x}_0)
    = \lim_{h \to 0} \frac{D^{z-1}_{\boldsymbol{v}} f_{NTK}(\boldsymbol{x}_1) - D^{z-1}_{\boldsymbol{v}} f_{NTK}(\boldsymbol{x}_0)}{h}.
\end{align}
Using the definition of $f_{NTK}$ we can expand the numerator of equation (43) for the first directional derivative:
\begin{align}
    & f_{NTK}(\boldsymbol{\hat{x}}_1) - f_{NTK}(\boldsymbol{\hat{x}}_0) \\
    &= \boldsymbol{\beta}_{NTK}^\top\
    \big(
        \boldsymbol{\hat{x}}_1 \cdot \mathbb{I}_1^{(k)}
        - \boldsymbol{\hat{x}}_0 \cdot \mathbb{I}_0^{(k)},
        {\boldsymbol{w}^{(k)}}^\top \boldsymbol{\hat{x}}_1 \cdot \mathbb{I}_1^{(k)}
        - {\boldsymbol{w}^{(k)}}^\top \boldsymbol{\hat{x}}_0 \cdot \mathbb{I}_0^{(k)},
        \, \ldots
    \big),
\end{align}
where $\mathbb{I}_0^{(k)} = \mathbb{I} \left( {\boldsymbol{w}^{(k)}}^\top \boldsymbol{\hat{x}}_0 \ge 0 \right)$, $\mathbb{I}_1^{(k)} = \mathbb{I} \left( {\boldsymbol{w}^{(k)}}^\top \boldsymbol{\hat{x}}_1 \ge 0 \right)$, $...$ are defined for notational brevity. The numerator for the second directional derivative is similarly expanded, omitting an $h$ which has been factored out:
\begin{align}
    & (f_{NTK}(\boldsymbol{\hat{x}}_2) - f_{NTK}(\boldsymbol{\hat{x}}_1))-(f_{NTK}(\boldsymbol{\hat{x}}_1) - f_{NTK}(\boldsymbol{\hat{x}}_0)) \\
    &= \boldsymbol{\beta}_{NTK}^\top\
    \big(
        (\boldsymbol{\hat{x}}_2 \cdot \mathbb{I}_2^{(k)}
        - \boldsymbol{\hat{x}}_1 \cdot \mathbb{I}_1^{(k)})
        - (\boldsymbol{\hat{x}}_1 \cdot \mathbb{I}_1^{(k)}
        - \boldsymbol{\hat{x}}_0 \cdot \mathbb{I}_0^{(k)}), \\
        &
        ({\boldsymbol{w}^{(k)}}^\top \boldsymbol{\hat{x}}_2 \cdot \mathbb{I}_2^{(k)}
        - {\boldsymbol{w}^{(k)}}^\top \boldsymbol{\hat{x}}_1 \cdot \mathbb{I}_1^{(k)})
        - {\boldsymbol{w}^{(k)}}^\top \boldsymbol{\hat{x}}_1 \cdot \mathbb{I}_1^{(k)}
        - {\boldsymbol{w}^{(k)}}^\top \boldsymbol{\hat{x}}_0 \cdot \mathbb{I}_0^{(k)},
        \, \ldots
    \big)
    \\
    &= \boldsymbol{\beta}_{NTK}^\top\
    \big(
        \boldsymbol{\hat{x}}_2 \cdot \mathbb{I}_2^{(k)}
        - 2 \boldsymbol{\hat{x}}_1 \cdot \mathbb{I}_1^{(k)}
        + \boldsymbol{\hat{x}}_0 \cdot \mathbb{I}_0^{(k)} , \\
        &
        {\boldsymbol{w}^{(k)}}^\top \boldsymbol{\hat{x}}_2 \cdot \mathbb{I}_2^{(k)}
        - 2 {\boldsymbol{w}^{(k)}}^\top \boldsymbol{\hat{x}}_1 \cdot \mathbb{I}_1^{(k)}
        + {\boldsymbol{w}^{(k)}}^\top \boldsymbol{\hat{x}}_0 \cdot \mathbb{I}_0^{(k)},
        \, \ldots
    \big),
\end{align}
and so forth. The point is that the $z$-th directional derivative of $f_{NTK}$ will contain the terms $\boldsymbol{x}_{0}, \boldsymbol{x}_{1}, ..., \boldsymbol{x}_{z}$ where $\boldsymbol{x}_{z} = \boldsymbol{x}_{0} + zh \boldsymbol{v}$ where we repeatedly differentiate along the same direction $\boldsymbol{v}$.

Next, let ${\Sigma}^{(z)}_{\boldsymbol{\hat{x}}, \mathbb{I}^{(k)}}$ be defined as:
\begin{align}
    {\Sigma}^{(z)}_{\boldsymbol{\hat{x}}, \mathbb{I}^{(k)}}
    = P_{z}^{(z)} \boldsymbol{\hat{x}}_z \mathbb{I}_z^{(k)}
    + P_{z-1}^{(z)} \boldsymbol{\hat{x}}_{z-1} \mathbb{I}_{z-1}^{(k)}
    + ... 
    + P_{0}^{(z)} \boldsymbol{\hat{x}}_0 \mathbb{I}_0^{(k)},
\end{align}
where the coefficients $P_{z}^{(z)}, P_{z-1}^{(z)}, ..., P_0^{(z)}$ represent the sign-alternating Pascal coefficients of the $z$-th line in a $0$-indexed Pascal triangle, e.g., $P_1^{(1)}=1$ and $P_0^{(1)}=-1$. We can now generally rewrite the $z$-th directional derivative of $f_{NTK}$ using equation (51) as: 
\begin{align}
    & D^{z}_{\boldsymbol{v}} f(\boldsymbol{\hat{x}}_0) = 
    \lim_{h \to 0}
    \frac{
    \boldsymbol{\beta}_{NTK}^\top
    \big(
        {\Sigma}^{(z)}_{\boldsymbol{\hat{x}}, \mathbb{I}^{(k)}},    
        {\boldsymbol{w}^{(k)}}^\top {\Sigma}^{(z)}_{\boldsymbol{\hat{x}}, \mathbb{I}^{(k)}}
        , \, ...
    \big)}
    {h^{z}}.
\end{align}
The last definition in preparation for the proof of Lemma 1 will be the sum ${\Sigma}^{(z)}_{\mathbb{I}^{(k)}}$, defined in terms of the indicators:
\begin{align}
    {\Sigma}^{(z)}_{\mathbb{I}^{(k)}}
    = P_{z}^{(z)} \mathbb{I}_z^{(k)}
    + P_{z-1}^{(z)} \mathbb{I}_{z-1}^{(k)}
    + ... 
    + P_{0}^{(z)} \mathbb{I}_0^{(k)}.
\end{align}

\begin{lemma}
The feature map of the $z$-th directional derivative of $f_{NTK}$ for any direction $\boldsymbol{v}_{0}$ can be expressed in terms of the $z$-th and $(z-1)$-th directional derivatives of the indicator for $\boldsymbol{v}_{0}$ such that:
\begin{align*}
    &D^{z}_{\boldsymbol{v}} f_{NTK}(\boldsymbol{x}_0)
    = \boldsymbol{\beta}_{\text{NTK}}^\top
    \left(
        \boldsymbol{\hat{x}}_0 \cdot D_{\boldsymbol{v}}^{z} \, \mathbb{I}^{(k)}
        - z \hat{\boldsymbol{v}} \cdot D_{\boldsymbol{v}}^{z-1} \, \mathbb{I}^{(k)}, \,
        {\boldsymbol{w}^{(k)}}^\top \boldsymbol{\hat{x}}_0 \cdot D_{\boldsymbol{v}}^{z} \, \mathbb{I}^{(k)}
        - z {\boldsymbol{w}^{(k)}}^\top \hat{\boldsymbol{v}} \cdot D_{\boldsymbol{v}}^{z-1} \, \mathbb{I}^{(k)}, 
        \, ...
    \right)
\end{align*}
\end{lemma}

\begin{proof}
The first term of ${\Sigma}^{(z)}_{\boldsymbol{\hat{x}}, \mathbb{I}^{(k)}}$ is $P_{z}^{(z)} \boldsymbol{\hat{x}}_z \mathbb{I}_z^{(k)}$. Since $P_{z}^{(z)}$ always lies on the left edge of the Pascal triangle, we always have $P_{z}^{(z)} \boldsymbol{\hat{x}}_z \mathbb{I}_z^{(k)} = \boldsymbol{\hat{x}}_z \mathbb{I}_z^{(k)}$. We use the trick
\begin{align}
    & \boldsymbol{\hat{x}}_z \mathbb{I}_z^{(k)} \\
    &= \boldsymbol{\hat{x}}_z 
        (
            \mathbb{I}_z^{(k)}
            + ({\Sigma}^{(z)}_{\mathbb{I}^{(k)}} - \mathbb{I}_z^{(k)})
            - ({\Sigma}^{(z)}_{\mathbb{I}^{(k)}} - \mathbb{I}_z^{(k)})
        ) \\
    &= \boldsymbol{\hat{x}}_z \cdot {\Sigma}^{(z)}_{\mathbb{I}^{(k)}}
    - \boldsymbol{\hat{x}}_z \cdot 
    (
    {\Sigma}^{(z)}_{\mathbb{I}^{(k)}}
    - \mathbb{I}_z^{(k)}
    )
\end{align}
so that
\begin{align}
    & {\Sigma}^{(z)}_{\boldsymbol{\hat{x}}, \mathbb{I}^{(k)}} \\
    &= 
    \boldsymbol{\hat{x}}_z \cdot {\Sigma}^{(z)}_{\mathbb{I}^{(k)}}
    - \boldsymbol{\hat{x}}_z \cdot 
    (
        {\Sigma}^{(z)}_{\mathbb{I}^{(k)}}
        - \mathbb{I}_z^{(k)}
    )
    +
    (
        {\Sigma}^{(z)}_{\boldsymbol{\hat{x}}, \mathbb{I}^{(k)}}
        - \boldsymbol{\hat{x}}_z \mathbb{I}_z^{(k)}
    ) 
    \\
    &= 
    \boldsymbol{\hat{x}}_z \cdot {\Sigma}^{(z)}_{\mathbb{I}^{(k)}}
    + {\Sigma}^{(z)}_{\boldsymbol{\hat{x}}, \mathbb{I}^{(k)}}
    - \boldsymbol{\hat{x}}_z \cdot {\Sigma}^{(z)}_{\mathbb{I}^{(k)}}
    \\
    &= \boldsymbol{\hat{x}}_z \cdot {\Sigma}^{(z)}_{\mathbb{I}^{(k)}} 
    + (
            P_{z-1}^{(z)} \boldsymbol{\hat{x}}_{z-1} \mathbb{I}_{z-1}^{(k)} -
            P_{z-1}^{(z)} \boldsymbol{\hat{x}}_{z} \mathbb{I}_{z-1}^{(k)}
        )
    + ...
    + (
            P_{0}^{(z)} \boldsymbol{\hat{x}}_{0} \mathbb{I}_{0}^{(k)} -
            P_{0}^{(z)} \boldsymbol{\hat{x}}_{z} \mathbb{I}_{0}^{(k)}
        )
    \\
    &= \boldsymbol{\hat{x}}_z \cdot {\Sigma}^{(z)}_{\mathbb{I}^{(k)}}
    + P_{z-1}^{(z)} \mathbb{I}_{z-1}^{(k)} (\boldsymbol{\hat{x}}_{z-1} - \boldsymbol{\hat{x}}_{z})
    + ...
    + P_{0}^{(z)} \mathbb{I}_{0}^{(k)} (\boldsymbol{\hat{x}}_{0} - \boldsymbol{\hat{x}}_{z})
    \\
    &= \boldsymbol{\hat{x}}_z \cdot {\Sigma}^{(z)}_{\mathbb{I}^{(k)}}
    + P_{z-1}^{(z)} \mathbb{I}_{z-1}^{(k)} ([-h \boldsymbol{v} \,|\, 0])
    + ...
    + P_{0}^{(z)} \mathbb{I}_{0}^{(k)} ([-zh \boldsymbol{v} \,|\, 0])
    \\
     &= \boldsymbol{\hat{x}}_z \cdot {\Sigma}^{(z)}_{\mathbb{I}^{(k)}}
     - \left(
            \sum_{i=0}^{z-1} P_{i}^{(z)} \mathbb{I}_i^{(k)} (z-i) h [\boldsymbol{v} \,|\, 0]
        \right)
    \\
     &= \boldsymbol{\hat{x}}_z \cdot {\Sigma}^{(z)}_{\mathbb{I}^{(k)}}
     - \left(
            h [\boldsymbol{v} \,|\, 0]
            \sum_{i=0}^{z-1} P_{i}^{(z)} \mathbb{I}_i^{(k)} (z-i)
        \right)
    \\
    &= \boldsymbol{\hat{x}}_z \cdot {\Sigma}^{(z)}_{\mathbb{I}^{(k)}}
     - \left(
            h [\boldsymbol{v} \,|\, 0] 
            \sum_{i=0}^{z-1} z P_{i}^{(z-1)} \mathbb{I}_i^{(k)}
        \right)
    \\
    &= \boldsymbol{\hat{x}}_z \cdot {\Sigma}^{(z)}_{\mathbb{I}^{(k)}}
     - z h [\boldsymbol{v} \,|\, 0] \cdot {\Sigma}^{(z-1)}_{\mathbb{I}^{(k)}}
\end{align}
where we use the algebraic trick of lines (54)-(56) to get a secondary trick on line (59), which would be otherwise more difficult to see. Then, lines (60)-(62) follow from the definitions of lines (51) and (53). Another critical step in the proof sequence is the penultimate equality which realizes the equivalence between $P_{i}^{(z)} (z-i)$ and $z P_{i}^{(z-1)}$ for $i = 0, ..., z-1$. This equivalence finds a correspondence from a coefficient on the $z$-th line of the Pascal triangle to the number on the left in the previous $(z-1)$-th line. And since the equivalence is defined for $i=0,...,z-1$, it is well-defined because the correspondence from a coefficient $P_{i}^{(z)}$ to the preceding coefficient $P_{i}^{(z-1)}$ on the left of the triangle is only undefined when $i = z$ which would be out of bounds with respect to an indexing error on the $(z-1)$-th line. Also note that this equivalence implicitly requires $z \ge 1$ since we are computing derivatives.

The significance of the result on line (66) is that we can reformulate equation (52) to be expressed in terms of the definition (53), effectively re-expressing the binomial expansion coefficients, which comes from the limit definition of the directional derivative of $f_{NTK}$, in terms of the indicators. And since we only manipulated equation (51), the limit is now taken with respect to the indicators. The point is that we can now write the $z$-th derivative of $f_{NTK}$ in terms of the $z$-th and $(z-1)$-th directional derivatives of $\mathbb{I}$:
\begin{align}
    & D_{\boldsymbol{v}}^{z} f(\boldsymbol{\hat{x}}_{0}) \\
    &= 
    \lim_{h \to 0} \frac{\boldsymbol{\beta}_{\text{NTK}}^\top}{h^z}
    \big(
        \boldsymbol{\hat{x}}_z \cdot {\Sigma}^{(z)}_{\mathbb{I}^{(k)}}
        - z h \hat{\boldsymbol{v}} \cdot {\Sigma}^{(z-1)}_{\mathbb{I}^{(k)}}, \,
        {\boldsymbol{w}^{(k)}}^\top(
            \boldsymbol{\hat{x}}_z \cdot {\Sigma}^{(z)}_{\mathbb{I}^{(k)}}
            - z h \hat{\boldsymbol{v}} \cdot {\Sigma}^{(z-1)}_{\mathbb{I}^{(k)}}
        ), \, ...
    \big) \\
    &= \lim_{h \to 0} \boldsymbol{\beta}_{\text{NTK}}^\top
    \big(
        \boldsymbol{\hat{x}}_z \cdot ({\Sigma}^{(z)}_{\mathbb{I}^{(k)}}/h^z)
        - z \hat{\boldsymbol{v}} \cdot ({\Sigma}^{(z-1)}_{\mathbb{I}^{(k)}}/h^{z-1}), \,
        {\boldsymbol{w}^{(k)}}^\top (
            \boldsymbol{\hat{x}}_z \cdot ({\Sigma}^{(z)}_{\mathbb{I}^{(k)}}/h^z) 
            - z \hat{\boldsymbol{v}} \cdot ({\Sigma}^{(z-1)}_{\mathbb{I}^{(k)}}/h^{z-1})
        ), \, ...
    \big) \\
    &= \boldsymbol{\beta}_{\text{NTK}}^\top
    \left(
        \boldsymbol{\hat{x}}_0 \cdot D_{\boldsymbol{v}}^{z} \, \mathbb{I}^{(k)}
        - z \hat{\boldsymbol{v}} \cdot D_{\boldsymbol{v}}^{z-1} \, \mathbb{I}^{(k)}, \,
        {\boldsymbol{w}^{(k)}}^\top \boldsymbol{\hat{x}}_0 \cdot D_{\boldsymbol{v}}^{z} \, \mathbb{I}^{(k)}
        - z {\boldsymbol{w}^{(k)}}^\top \hat{\boldsymbol{v}} \cdot D_{\boldsymbol{v}}^{z-1} \, \mathbb{I}^{(k)}, 
        \, ...
    \right)
\end{align}
where line (70) follows from the limit definition of the directional derivative of the indicator evaluated at $\boldsymbol{\hat{x}}_0$ which completes our proof of Lemma 1.
\end{proof}
Let us look closer at line (70). It is well known that the derivative of the indicator (Heaviside function) $\mathbb{I}$ - or any such step function for this matter - does not classically have a well-defined derivative. This fact makes the analysis beyond equation (70) difficult because we are interested in the derivative of the indicator evaluated at $\boldsymbol{x}_{0} = \boldsymbol{0}$, which is precisely where the discontinuity exists.

Fortunately, we have a workaround. By generalizing the notion of the indicator's derivative, we can consider the distributional derivative of the indicator, which is the Dirac-delta function (impulse spike located at $\boldsymbol{x}_{0} = \boldsymbol{0}$). This is a similar workaround to how we pseudo-inverted the otherwise singular constant matrix $\boldsymbol{J}$ from equation (27) by generalizing the notion of the matrix inverse. Using chain rule, the directional derivative of $\mathbb{I}$ evaluted at $\boldsymbol{x}_{0} = \boldsymbol{0}$ is:
\begin{align}
    & D_{\boldsymbol{v}}^{z} \, \mathbb{I} \left( {\boldsymbol{w}}^\top \boldsymbol{\hat{x}}_{0} \ge 0 \right)
    = {\langle \boldsymbol{\check{w}}, \boldsymbol{v} \rangle}^z \cdot \delta^{(z-1)} (\boldsymbol{w}_{d+1}).
\end{align}
Equation (71) gives us a cleaner expression for the $z$-th derivative of $f_{NTK}$ by the sifting property of the Dirac-delta. Continuing from equation (70), we use line (71) to get:
\begin{align}
    &= \int
        (\boldsymbol{\beta}_{\boldsymbol{w}}^{1})_{d+1}
        {\langle \boldsymbol{\check{w}}, \boldsymbol{v} \rangle}^z 
        \cdot \delta^{(z-1)} (\boldsymbol{w}_{d+1})
       d\mathbb{P}(\boldsymbol{w})
    -  \int
        z 
        \langle \boldsymbol{\beta}_{\boldsymbol{w}}^{1} \hat{\boldsymbol{v}} \rangle
        {\langle \boldsymbol{\check{w}}, \boldsymbol{v} \rangle}^{z-1}
        \cdot \delta^{(z-2)} (\boldsymbol{w}_{d+1})
       d\mathbb{P}(\boldsymbol{w}) \\
    &+ \int 
        {\boldsymbol{\beta}_{\boldsymbol{w}}^{2}}
        {\boldsymbol{w}}_{d+1}
        {\langle \boldsymbol{\check{w}}, \boldsymbol{v} \rangle}^z 
        \cdot \delta^{(z-1)} (\boldsymbol{w}_{d+1})
       d\mathbb{P}(\boldsymbol{w})
    -
       \int
        z
        {\boldsymbol{\beta}_{\boldsymbol{w}}^{2}}
        {\langle \boldsymbol{\check{w}}, \boldsymbol{v} \rangle}^{z} 
        \cdot \delta^{(z-2)} (\boldsymbol{w}_{d+1})
       d\mathbb{P}(\boldsymbol{w}) \\
    &=  (-1)^{z-1}
        {\langle \boldsymbol{\check{w}}, \boldsymbol{v} \rangle}^z
        \left[
            \frac{\partial^{z-1}}{\partial \boldsymbol{w}_{d+1}^{z-1}}
            (\boldsymbol{\beta}_{\boldsymbol{w}}^{1})_{d+1}
        \right]_{{\boldsymbol{w}}_{d+1}=0}
    +   (-1)^{z-1}
        z
        {\langle \boldsymbol{\check{w}}, \boldsymbol{v} \rangle}^{z-1}
        \left[
            \frac{\partial^{z-2}}{\partial \boldsymbol{w}_{d+1}^{z-2}}
            \langle \boldsymbol{\beta}_{\boldsymbol{w}}^{1} \hat{\boldsymbol{v}} \rangle
        \right]_{{\boldsymbol{w}}_{d+1}=0} \\
    &+  (-1)^{z-1}
        {\langle \boldsymbol{\check{w}}, \boldsymbol{v} \rangle}^z
        \left[
            \frac{\partial^{z-1}}{\partial \boldsymbol{w}_{d+1}^{z-1}}
            {\boldsymbol{\beta}_{\boldsymbol{w}}^{2}} 
            {\boldsymbol{w}}_{d+1}
        \right]_{{\boldsymbol{w}}_{d+1}=0}
    +   (-1)^{z-1}
        z
        {\langle \boldsymbol{\check{w}}, \boldsymbol{v} \rangle}^{z}
        \left[
            \frac{\partial^{z-2}}{\partial \boldsymbol{w}_{d+1}^{z-2}}
            {\boldsymbol{\beta}_{\boldsymbol{w}}^{2}}
        \right]_{{\boldsymbol{w}}_{d+1}=0}
\end{align}
which rewrites the derivative of $f_{NTK}$ in terms of the high derivatives of the beta components. The check notation $\boldsymbol{\check{w}}$ emphasizes that the direction vector no longer depends on the bias component, $\boldsymbol{w}_{d+1}$. Actually, the equivalence between $\boldsymbol{\check{w}}^\top \boldsymbol{v}$ and $\boldsymbol{w}^\top \boldsymbol{\hat{v}}$ is a useful consideration.

At this point, the emergence of the Dirac impulse suggests that nonlinearity may be preserved in high orders. Nonlinearity, if preserved, depends closely on the relationship between directions $\boldsymbol{w}$ and $\boldsymbol{\hat{v}}$ as well as the the derivative of the $\boldsymbol{\beta}_{NTK}$ components with respect to $\boldsymbol{w}_{d+1}$. However, the precise degree of nonlinearity remains unknown because we have not yet accounted for the influence of $\varphi^\infty$. The upcoming section demonstrates that solving for the derivative of the $\boldsymbol{\beta}_{NTK}$ components simultaneously accounts for the positions of $\varphi^\infty$.

\subsection{Proof of Lemma 2}

Following lines (74)-(75), we must solve for the partial derivatives of the beta components. We will first consider the terms with a dependence on $\boldsymbol{w}$; Recall the forms of the beta components from section 3 equations (41)-(42):
\begin{align*}
    &\boldsymbol{\beta}^1_{\boldsymbol{w}} =
    \mathbb{I} \left( {\boldsymbol{w}}^\top (-\boldsymbol{\hat{v}}) \geq 0 \right)
    \cdot 
    \left(
        C(t, \delta, \kappa)
        \sum_{j=1}^n g(\boldsymbol{\hat{x}}_i^\infty)
        \sum_{i=1}^n \boldsymbol{\hat{x}}_i^\infty
        +
        \frac{1}{\delta}
        \sum_{i=1}^n \boldsymbol{\hat{x}}_i^\infty g(\boldsymbol{\hat{x}}_i^\infty)
    \right) \\
    &\boldsymbol{\beta}^2_{\boldsymbol{w}} =
    \mathbb{I} \left( {\boldsymbol{w}}^\top (-\boldsymbol{\hat{v}}) \geq 0 \right)
    \cdot 
    \left(
        C(t, \delta, \kappa)
        \sum_{j=1}^n g(\boldsymbol{\hat{x}}_i^\infty)
        \sum_{i=1}^n \boldsymbol{w}^\top \boldsymbol{\hat{x}}_i^\infty
        +
        \frac{1}{\delta}
        \sum_{i=1}^n \boldsymbol{w}^\top \boldsymbol{\hat{x}}_i^\infty g(\boldsymbol{\hat{x}}_i^\infty)
    \right).
\end{align*}
Firstly, the partial derivative of the indicator is a trivial analysis. We recall that the distributional derivative of the indicator is the Dirac-delta function. And, since the bias component of a direction vector $\boldsymbol{\hat{v}}_{d+1}$ is 0, it is not difficult to see that the $z$-th partial derivative of the indicator is $0$ with respect to $\boldsymbol{w}_{d+1}$ for all $z \ge 1$.
\begin{align}
    & \frac{\partial}{\partial \boldsymbol{w}_{d+1}} \mathbb{I} \left( {\boldsymbol{w}}^\top (-\boldsymbol{\hat{v}}) \geq 0 \right)
    = \delta ({\boldsymbol{w}}^\top (-\boldsymbol{\hat{v}})) \cdot (-\boldsymbol{\hat{v}}_{d+1})
\end{align}
Secondly, the partial derivative of the dot product $\boldsymbol{w}^\top \boldsymbol{\hat{x}}_i^\infty$ is also trivial to solve. For clarity, we apply the definition of $\varphi^\infty$ before computing the partial derivative. Since the bias components of a data point and direction vector are $1$ and $0$ respectively, it is clear to see that the partial derivative equals $1$. 
\begin{align}
    &\frac{\partial}{\partial \boldsymbol{w}_{d+1}} {\boldsymbol{w}}^\top \boldsymbol{\hat{x}}^\infty_i \\
    &= \frac{\partial}{\partial \boldsymbol{w}_{d+1}} ({\boldsymbol{w}}^\top \boldsymbol{\hat{x}}_i - t({\boldsymbol{w}}^\top \boldsymbol{\hat{v}})) \\
    &= (\boldsymbol{\hat{x}}_i)_{d+1} -t\boldsymbol{\hat{v}}_{d+1}
\end{align}
Lastly, we want to discover the partial derivative of the constant $C(t, \delta, \kappa)$. Recalling its definition, $\kappa$ is the only term in $C(t, \delta, \kappa)$ that depends on $\boldsymbol{w}$. We find that the $z$-th derivative of $\kappa$ is $0$ for any $z \ge 1$:
\begin{align}
    & \frac{\partial^{z} \kappa}{\partial \boldsymbol{w}^{z}_{d+1}} \\
    &= \frac{\partial^{z}}{\partial \boldsymbol{w}^{z}_{d+1}}
    \int 
        (\boldsymbol{\hat{v}}^2 + ({\boldsymbol{w}}^\top \boldsymbol{\hat{v}})^2) 
        \cdot \mathbb{I} \left( {\boldsymbol{w}}^\top (-\boldsymbol{\hat{v}}) \geq 0 \right)
    d \mathbb{P}(\boldsymbol{w}) \\
    &= \int \frac{\partial^{z}}{\partial \boldsymbol{w}^{z}_{d+1}}
        (\boldsymbol{\hat{v}}^2 + ({\boldsymbol{w}}^\top \boldsymbol{\hat{v}})^2) 
        \cdot \mathbb{I} \left( {\boldsymbol{w}}^\top (-\boldsymbol{\hat{v}}) \geq 0 \right)
    d \mathbb{P}(\boldsymbol{w}) \\
    &= \int
        \frac{\partial^{z-1}}{\partial \boldsymbol{w}^{z-1}_{d+1}}
        \left(
            2\left( {\boldsymbol{w}}^\top \boldsymbol{\hat{v}} \right) 
            \mathbb{I} \left( {\boldsymbol{w}}^\top (-\boldsymbol{\hat{v}}) \geq 0 \right)
            \boldsymbol{\hat{v}}_{d+1}
            - \left( \boldsymbol{\hat{v}}^2 + ({\boldsymbol{w}}^\top \boldsymbol{\hat{v}})^2 \right)
            \delta \left( {\boldsymbol{w}}^\top (-\boldsymbol{\hat{v}}) \right)
            \boldsymbol{\hat{v}}_{d+1}
        \right)
    d \mathbb{P}(\boldsymbol{w}),
\end{align}
where the third equality is $0$ from the fact that $\boldsymbol{\hat{v}}_{d+1}$ is $0$. We are now prepared to differentiate the beta components. 

\begin{lemma}
The components of the NTK representation coefficient $\boldsymbol{\beta}_{NTK}$ induced by a training input set $\varphi^\infty = \{ \boldsymbol{x}_{i}^\infty \}_{i=1}^{n}$ where $\boldsymbol{x}_{i}^\infty = \boldsymbol{x}_i - t \boldsymbol{v}_{\varphi}$ for some $\boldsymbol{x}_i \in \mathcal{X}$ and any direction $\boldsymbol{v}_{\varphi}$ are constant with respect to the bias component of any given feature direction $\boldsymbol{w}_{d+1}$ such that:
\begin{align*}
    \frac{\partial^z \boldsymbol{\beta}^1_{\boldsymbol{w}}}{\partial \boldsymbol{w}^z_{d+1}},
    \frac{\partial^z \boldsymbol{\beta}^2_{\boldsymbol{w}}}{\partial \boldsymbol{w}^z_{d+1}}
    = 0 
    \; \text{for all} \; z \ge 1.
\end{align*}
\end{lemma}

\begin{proof}
Differentiating the first beta component is relatively straightforward. By product rule, we analyze the derivative of the indication on the LHS and the sum on the RHS. We already know that the derivative of the indicator for a training point induced by $\varphi^\infty$ is $0$. We also know that the derivative of kappa is $0$. And, since no other terms depend on $\boldsymbol{w}$ the $z$-th derivative of the first beta component with respect to $\boldsymbol{w}_{d+1}$ is simply $0$ for all $z \ge 1$:
\begin{align}
    \frac{\partial \boldsymbol{\beta}^1_{\boldsymbol{w}}}{\partial \boldsymbol{w}_{d+1}}
    &= \left(
        \sum_{j=1}^n g(\boldsymbol{\hat{x}}_j^\infty)\,
        \sum_{i=1}^n \boldsymbol{\hat{x}}_i^\infty\,
        \frac{\partial C(t, \delta, \kappa)}{\partial \boldsymbol{w}_{d+1}}
    \right)
    \cdot 
    \mathbb{I}\!\left( \boldsymbol{w}^\top (-\boldsymbol{\hat{v}}) \ge 0 \right).
\end{align}
where
\begin{align}
    &   \frac{\partial C(t, \delta, \kappa)}{\partial \boldsymbol{w}_{d+1}}
    =  (t^2 \kappa)(n \delta t^2 \frac{\partial \kappa}{\partial \boldsymbol{w}_{d+1}}) 
        + (t^2 \frac{\partial \kappa}{\partial \boldsymbol{w}_{d+1}})(\delta(n \kappa t^2 + \delta)).
\end{align}
Similarly, we differentiate the second beta component by product rule. We observe the summation on the RHS where the dependence on $\boldsymbol{w}$ is more elaborate. Using equation (79) we can see that the derivative of the dot product in the second term of the summation reduces to $1$. Then, for the first term, we once again leverage equation (79) and the fact that the derivative of kappa is $0$ to discover by a straightforward algebraic manipulation that the $z$-th derivative of the second beta component approaches $0$ for all $z \ge 1$:
\begin{align}
    & \frac{\partial \boldsymbol{\beta}^2_{\boldsymbol{w}}}{\partial \boldsymbol{w}_{d+1}} \\
    &=  \left(
            C(t, \delta, \kappa)\,
            \sum_{j=1}^n g(\boldsymbol{\hat{x}}_i^\infty)\,
            n\,
            +
            \frac{1}{\delta}\,
            \sum_{i=1}^n g(\boldsymbol{\hat{x}}_i^\infty)
        \right) 
        \cdot
        \mathbb{I} \left( {\boldsymbol{w}}^\top (- \boldsymbol{\hat{v}}) \geq 0 \right) \\
    &=  \left(
            C(t, \delta, \kappa)\,
            \sum_{j=1}^n g(\boldsymbol{\hat{x}}_i^\infty)\,
            n\,
            +
            \frac{1}{\delta}\,
            \sum_{i=1}^n g(\boldsymbol{\hat{x}}_i^\infty)
        \right) 
        \mathbb{I} \left( {\boldsymbol{w}}^\top (- \boldsymbol{\hat{v}}) \geq 0 \right) \\
    &=  \left(
            - \frac{n \kappa g_{sum} t^2}{\delta (n \kappa t^2 + \delta)}\,
            + \frac{1}{\delta} \sum^n_{i=1} g(\boldsymbol{\hat{x}}^{\infty}_i) 
        \right) 
        \mathbb{I} \left( {\boldsymbol{w}}^\top (- \boldsymbol{\hat{v}}) \geq 0 \right) \\
    &=  \left(
            \frac{g_{sum}}{\delta}\,
            - \frac{n \kappa g_{sum} t^2}{\delta (n \kappa t^2 + \delta)} 
        \right) 
        \mathbb{I} \left( {\boldsymbol{w}}^\top (- \boldsymbol{\hat{v}}) \geq 0 \right) \\
    &=  \frac{g_{sum} \delta}{\delta (n \kappa t^2 + \delta)}\,
        \mathbb{I} \left( {\boldsymbol{w}}^\top (- \boldsymbol{\hat{v}}) \geq 0 \right) \\
    &=  \frac{g_{sum}}{n \kappa t^2 + \delta}\,
        \mathbb{I} \left( {\boldsymbol{w}}^\top (- \boldsymbol{\hat{v}}) \geq 0 \right).
\end{align}
Inspecting the final equality, it is not difficult to see that the first derivative of the second beta component approaches $0$ as $\delta \to 0^+$ and $t \to \infty$.  Furthermore, since the derivatives of the indicator and kappa are both $0$, it is clear to see that by chain rule, the second derivative of the second beta component is also $0$. Therefore, the $z$-th derivative of the second beta component is $0$ for all $z \ge 1$. This completes our proof of Lemma 2.
\end{proof}

\subsection{Proof of Theorem 1}

\begin{theorem}
An over-parameterized two-layer ReLU MLP $f_{NTK} : \mathbb{R}^{d} \to \mathbb{R}$ that is trained on a labeled set $\{ (\boldsymbol{x}_{i}^\infty, {y}_{i}^\infty) \}_{i=1}^{n}$ with $\boldsymbol{x}_{i}^\infty = \boldsymbol{x}_i - t \boldsymbol{v}_{\varphi}$ for $\boldsymbol{x}_i \in \mathcal{X}$ and any direction $\boldsymbol{v}_{\varphi}$ in the NTK regime minimizing squared loss will converge to a quadratic extrapolator when evaluated at a point near the origin $\boldsymbol{0}$ as $t \to \infty$.
\end{theorem}

\begin{proof}
Under the definition of $\varphi^\infty$, Lemma 2 states that $\frac{\partial^{z} \boldsymbol{\beta}^{1}_{\boldsymbol{w}}}{\partial {\boldsymbol{w}}_{d+1}^{z}}$ and $\frac{\partial^{z} \boldsymbol{\beta}^{2}_{\boldsymbol{w}}}{\partial {\boldsymbol{w}}_{d+1}^{z}}$ are $0$ for orders $z \ge 1$. But since Lemma 1 shows that $D^{z}_{\boldsymbol{v}_{\boldsymbol{0}}} f_{NTK}$ for any direction $\boldsymbol{v}_{\boldsymbol{0}}$ actually depends on the lower ordered $(z-1)$-th and $(z-2)$-th derivatives $\frac{\partial^{z-1} \boldsymbol{\beta}^{1}_{\boldsymbol{w}}}{\partial {\boldsymbol{w}}_{d+1}^{z-1}}$, $\frac{\partial^{z-2} \boldsymbol{\beta}^{1}_{\boldsymbol{w}}}{\partial {\boldsymbol{w}}_{d+1}^{z-2}}$, $\frac{\partial^{z-1} \boldsymbol{\beta}^{2}_{\boldsymbol{w}}}{\partial {\boldsymbol{w}}_{d+1}^{z-1}}$, and $\frac{\partial^{z-2} \boldsymbol{\beta}^{2}_{\boldsymbol{w}}}{\partial {\boldsymbol{w}}_{d+1}^{z-2}}$, it is not difficult to see that the third and all higher order derivatives are automatically $0$. Then, taking $z=1$ we simplify equation (74) to get an examinable form of the first derivative:
\begin{align*}
    &   D_{\boldsymbol{v}_{\boldsymbol{0}}} f_{NTK}(\boldsymbol{\hat{0}}) \\
    &=  
        {\langle \boldsymbol{\check{w}}, \boldsymbol{v} \rangle}
        \left[
            (\boldsymbol{\beta}_{\boldsymbol{w}}^{1})_{d+1}
        \right]_{\boldsymbol{w}_{d+1}=0}
    -  \int
        {\boldsymbol{\beta}_{\boldsymbol{w}}^{1}}^\top 
        \hat{\boldsymbol{v}} \cdot
        \mathbb{I} \left( \boldsymbol{w}_{d+1} \ge 0 \right)
       d\mathbb{P}(\boldsymbol{w})
    -
       \int
        {\boldsymbol{\beta}_{\boldsymbol{w}}^{2}}
        {\boldsymbol{w}}^\top
        \hat{\boldsymbol{v}} \cdot
        \mathbb{I} \left( \boldsymbol{w}_{d+1} \ge 0 \right)
       d\mathbb{P}(\boldsymbol{w}).
\end{align*}
But more interestingly, we take $z=2$ and simplify equation (75) for the second derivative:
\begin{align*}
    & D^{2}_{\boldsymbol{v}_{\boldsymbol{0}}} f_{NTK}(\boldsymbol{\hat{0}}) \\
    &= 
    -   {\langle \boldsymbol{\check{w}}, \boldsymbol{v} \rangle}^2
        \left[
            \frac{\partial}{\partial \boldsymbol{w}_{d+1}}
            (\boldsymbol{\beta}_{\boldsymbol{w}}^{1})_{d+1}
        \right]_{{\boldsymbol{w}}_{d+1}=0}
    -   2
        {\langle \boldsymbol{\check{w}}, \boldsymbol{v} \rangle}
        \left[
            \langle \boldsymbol{\beta}_{\boldsymbol{w}}^{1} \hat{\boldsymbol{v}} \rangle
        \right]_{{\boldsymbol{w}}_{d+1}=0} \\
    &-  {\langle \boldsymbol{\check{w}}, \boldsymbol{v} \rangle}^2
        \left[
            \frac{\partial}{\partial \boldsymbol{w}_{d+1}}
            {\boldsymbol{\beta}_{\boldsymbol{w}}^{2}} 
            {\boldsymbol{w}}_{d+1}
        \right]_{\boldsymbol{w}_{d+1}=0}
    -   2
        {\langle \boldsymbol{\check{w}}, \boldsymbol{v} \rangle}^{2}
        \left[
            {\boldsymbol{\beta}_{\boldsymbol{w}}^{2}}
        \right]_{\boldsymbol{w}_{d+1}=0} \\
    &= 
    -   2
        {\langle \boldsymbol{\check{w}}, \boldsymbol{v} \rangle}
        \left[
            \langle \boldsymbol{\beta}_{w}^{1} \hat{\boldsymbol{v}} \rangle
        \right]_{\boldsymbol{w}_{d+1}=0}
    -  {\langle \boldsymbol{\check{w}}, \boldsymbol{v} \rangle}^2
        \left[
            {\boldsymbol{\beta}_{\boldsymbol{w}}^{2}} 
        \right]_{\boldsymbol{w}_{d+1}=0}
    -   2
        {\langle \boldsymbol{\check{w}}, \boldsymbol{v} \rangle}^{2}
        \left[
            {\boldsymbol{\beta}_{\boldsymbol{w}}^{2}}
        \right]_{\boldsymbol{w}_{d+1}=0} \\
    &= 
    -   2
        {\langle \boldsymbol{\check{w}}, \boldsymbol{v} \rangle}
        \left[
            \langle \boldsymbol{\beta}_{\boldsymbol{w}}^{1} \hat{\boldsymbol{v}} \rangle
        \right]_{\boldsymbol{w}_{d+1}=0}
    -   3
        {\langle \boldsymbol{\check{w}}, \boldsymbol{v} \rangle}^{2}
        \left[
            {\boldsymbol{\beta}_{\boldsymbol{w}}^{2}}
        \right]_{\boldsymbol{w}_{d+1}=0},
\end{align*}
to see a great dependence in the final equality on the beta components and dot product between any particular $\boldsymbol{w}$ and direction of evaluation $\boldsymbol{v}_{\boldsymbol{0}}$. Thus, for the special case of a training input set $\varphi^\infty$ whose members are located far from the origin, the regressor becomes a quadratic extrapolator when evaluated near the origin.
\end{proof}

\end{document}